\documentclass[12pt]{amsart}
\usepackage[left=15mm, right=15mm, top=10mm, bottom=15mm]{geometry}

\usepackage[english]{babel}
\usepackage[T1]{fontenc}
\usepackage[utf8]{inputenc}

\theoremstyle{plain}
    \newtheorem{theorem}{Theorem}[section]

\theoremstyle{definition}
    \newtheorem{remark}[theorem]{Remark}

\usepackage[
    draft = false,
    unicode = true,
    colorlinks = true,
    allcolors = blue,
    hyperfootnotes = true,
    citecolor = red
]{hyperref}
\usepackage{amsfonts,amsmath,amssymb,amscd,amsthm,url}
\usepackage[inline]{enumitem}
\usepackage{graphicx,epsf,afterpage, wrapfig}
\usepackage{soul}
\usepackage{multicol}
\usepackage{array}
\usepackage{braket}
\usepackage{epigraph}
\usepackage{rotating, floatflt}
\usepackage{xstring} % Пакет для надежной работы со строками

\usepackage[
backend=biber,
style=numeric-comp,
sorting=none,
giveninits=true
]{biblatex}
\usepackage{csquotes}

\usepackage{xcolor}

\usepackage{ulem}

\usepackage{tikz}
\usetikzlibrary{positioning, calc, intersections, through, arrows, matrix, chains, math}
\usetikzlibrary{arrows.meta}
\usetikzlibrary{shadows}
\usetikzlibrary{graphs}
\usetikzlibrary{graphs.standard}
\usetikzlibrary{patterns}
\usetikzlibrary{decorations.pathreplacing,calligraphy,backgrounds}
\DeclareFieldFormat{usera}{\href{https://arxiv.org/abs/#1}{\texttt{arXiv:#1}}}

\renewbibmacro*{finentry}{\printfield{usera}\newunit\finentry}

\renewbibmacro{in:}{%
  \iffieldundef{usera}
    {\printtext{\bibstring{in}\intitlepunct}}
    {}%
}

\newcommand\norm[1]{\ensuremath{\left\lVert#1\right\rVert}}
\newcommand\abs[1]{\ensuremath{\left\lvert#1\right\rvert}}

\DeclareMathOperator{\id}{id}

\DeclareMathOperator{\Ker}{Ker}

\DeclareMathOperator{\vspan}{span}

\newcommand{\Hcal}{\mathcal{H}}

\newcommand{\defing}[1]{\textbf{\emph{\mathversion{bold}#1}}}

\newcommand{\R}{\ensuremath{\mathbb{R}}}

\newcommand{\N}{\ensuremath{\mathbb{N}}}

\renewcommand{\geq}{\geqslant}

\newcounter{mcnt}

\newcounter{wordcnt}

\begin{document}

\title{Kolmogorov--Arnold stability for discontinuous functions}

\author{Sviatoslav V. Dzhenzher}

\begin{abstract}
    Here we investigate the stability of the Kolmogorov--Arnold representation theorem (KART) under adversarial reparameterisations of the hidden layer for multivariate discontinuous and unbounded functions. Our results provide a rigorous mathematical foundation for the structural robustness of modern deep learning architectures, such as Kolmogorov--Arnold Networks (KANs), under adversarial configurations.
\end{abstract}

\thanks{\hspace{-5mm}
S.\,V. Dzhenzher: sdjenjer@yandex.ru. orcid: 0009-0008-3513-4312
\\
% V.\,Zh. Sakbaev: fumi2003@mail.ru. orcid: 0000-0001-8349-1738
% \\
% V.\,Zh. Sakbaev: Keldysh Institute of Applied Mathematics of Russian Academy of Sciences 125047, Miusskaya pl. 4, Moscow, Russia
% \\
% All authors:
Moscow Institute of Physics and Technology 141701, Institutskii lane, 9, Dolgoprudny, Russia}

\maketitle
\thispagestyle{empty}

\noindent \emph{Keywords:} Kolmogorov--Arnold representation theorem, discontinuous functions, superposition of functions, Hilbert's 13th problem, adversarial robustness, Kolmogorov--Arnold Networks (KANs).

\vspace{3mm}
\noindent \emph{MSC 2020:}
Primary:
26B40, % — Representation and superposition of functions (Functions of several variables)
26A15; % — Real functions: Continuity and related questions (moduli of continuity, semicontinuity, discontinuities, etc.).
Secondary:
46B25, % - Classical Banach spaces in functional analysis.
47B38, % — Operators on function spaces 
68T07. % — Core models and architectures (Artificial neural networks and deep learning)

\section{Introduction}

It is quite clear that any multivariate polynomial can be represented as the superposition of a fixed (assuming that the number of variables and the degree of the polynomial are known a priori) number of univariate polynomials and addition.
The well celebrated Hilbert's 13th problem \cite{wiki-Hilb-13-problem, Morris-2021-hilbert13} asks whether the analogous result holds for multivariate continuous functions.
Originally, this problem was posed for the solutions of high-degree equations.
Precisely, in 1836, Hamilton showed \cite{Hamilton1836} that any seventh-degree equation can be reduced to an equation of the form
\[
    x^7 + ax^3 + bx^2 + cx + 1 = 0.
\]
Hilbert asked whether the solution \(x=x(a,b,c)\) of this equation can be expressed as the superposition of functions of two variables.
Of course, he expected the answer to be negative.
However, in 1956--57, for continuous functions all expectations were refuted by Kolmogorov and Arnold \cite{Kolmogorov56, Arnold57, Kolmogorov57}.
They showed that any multivariate continuous function can be represented by superpositions of univariate continuous functions and addition.
Nowadays, their result is known as the Kolmogorov--Arnold representation theorem (KART; see Theorem~\ref{t:kart}) or the Kolmogorov Superposition Theorem (this name is more convenient for computer science).

The relevance of KART to neural networks (NNs) was first observed by Hecht-Nielsen \cite{HechtNielsen87}.
He showed that any continuous multivariate function is implementable by three-layer feedforward NNs with continuous activation functions and real weights, and that any computable multivariate function is implementable by three-layer feedforward NNs with computable activation functions and computable real weights.
The ideas of the expressivity of functions were also expressed in \cite{Widrow-Lehr, Rumelhart-Hinton-Williams}.
There was a lot of debate \cite{Girosi-Pogio, Kurkova91} surrounding this topic regarding its real‑world applicability.
Finally, later works showed the great applicability of KART \cite{Maiorov-Pinkus, Freedman24}.

The recent resurgence of interest in KART has been largely fuelled by the introduction of Kolmogorov--Arnold Networks (KANs) \cite{KAN, KAN2, CKAN}, which position themselves as a promising alternative to traditional Multi-Layer Perceptrons (MLPs).
After that, in \cite{KAN-robust}, the practical robustness of KANs to black-box and white-box adversarial attacks was firstly studied.
However, while KANs offer notable advantages in terms of interpretability and accuracy, their \emph{mathematical} robustness under adversarial perturbations remains a critical open question.
In practice, NNs often encounter data with sharp transitions or unbounded behaviours, rendering the classical assumption of continuous target functions restrictive.
Such questions regarding KART were posed in \cite{DzhenzherFreedman-25}.
More specifically, the stability of KART under reparameterisations of the hidden layer was investigated.
It was assumed that the action of the reparameterisation, which may be considered as the adversarial homeomorphism, is known and that the output layer should be adjusted to it.
It was shown there that KART is stable under countable collections of homeomorphisms acting on the hidden layer.

In recent Ismailovs' works on KART and NNs \cite{Ismailova-Ismailov, Ismailov-2026}, the Kolmogorov--Arnold result was generalised to the case of discontinuous unbounded functions.
That framework assumes a static, error-free hidden layer.
In real-world computing, hidden layers are dynamic and prone to perturbations.
Our work bridges this gap by proving that Ismailovs' extensions are structurally stable under adversarial reparameterisations.
Strictly speaking, we state the stability of KART under countable collections of homeomorphisms for discontinuous unbounded functions.

The remainder of this paper is organised as follows.
Section~\ref{s:defs} establishes the necessary notation, recalls the foundational forms of KART, and states our main result (Theorem~\ref{t:kart-stab}).
Section~\ref{s:proof} provides the proof of the main result.
Section~\ref{s:concl} discusses the theoretical implications of our findings and outlines prominent directions for constructive and shift-invariant future research.

\section{Background and the main result}\label{s:defs}

We say that a continuous tuple \(\phi\colon [\,0,1\,]\to[\,0,1\,]^{2n+1}\) is the tuple \(\phi=(\phi_1,\ldots,\phi_{2n+1})\) of continuous functions \(\phi_i\colon[\,0,1\,]\to[\,0,1\,]\) for \(i=1,\ldots,2n+1\).

One of the forms of the Kolmogorov--Arnold representation theorems is as follows.
For details, see, for example, \cite{Lorentz-Golitschek-Makovoz, Hedberg-appx}.

\begin{theorem}[Kolmogorov, Arnold; 1956--57]\label{t:kart}
    Let $n>1$ be an integer. \\
    Then for any rationally independent \(\lambda_1,\ldots,\lambda_n\in\R\), \\
    there exists an ``inner'' continuous tuple \(\phi\colon[\,0,1\,]\to[\,0,1\,]^{2n+1}\) such that \\
    for any ``target'' continuous function \(f\colon[\,0,1\,]^n\to\R\), \\
    there exists an ``outer'' uniformly continuous function \(g\colon \R\to\R\) such that
    \[
        f(x) = \sum_{i=1}^{2n+1} g\left(\sum_{j=1}^n \lambda_j\phi_i(x_j)\right)
        \quad\text{for}\quad x=(x_1,\ldots,x_n)\in[\,0,1\,]^n.
    \]
\end{theorem}

The conventional choice of $\lambda_j$ is the square roots of pairwise distinct prime numbers.

Let $n>1$ be an integer.
Let \(\Hcal\) be a set of homeomorphisms \(\R^{2n+1}\to\R^{2n+1}\).
We say that \defing{KART is stable under \(\Hcal\)} if \\
for any rationally independent \(\lambda_1,\ldots,\lambda_n\in\R\), \\
there exists an ``inner'' continuous tuple \(\phi\colon[\,0,1\,]\to[\,0,1\,]^{2n+1}\) such that \\
for any ``target'' continuous function \(f\colon[\,0,1\,]^n\to\R\) and any ``adversarial'' \(h\in\Hcal\), \\
there exists an ``outer'' continuous function \(g\colon \R\to\R\) such that
\[
    f(x) = \sum_{i=1}^{2n+1} g\left(\sum_{j=1}^n \lambda_jh_i(\phi(x_j))\right)
    \quad\text{for}\quad x=(x_1,\ldots,x_n)\in[\,0,1\,]^n.
\]
For example, \cite{DzhenzherFreedman-25} states that KART is stable under any countable set of homeomorphisms.
Also, the stability of KART under the single-element \(\Hcal=\{\id\colon\R^{2n+1}\to\R^{2n+1}\}\) is the KART itself (sorry for the tautology).

Informally speaking, the following theorem states that the stability of KART is inherited in the case of discontinuous and unbounded functions.

\begin{theorem}\label{t:kart-stab}
    Let \(\Hcal\) be a set of homeomorphisms \(\R^{2n+1}\to\R^{2n+1}\) such that KART is stable under $\Hcal$.
    Then for any rationally independent \(\lambda_1,\ldots,\lambda_n\in\R\), \\
    there exists an ``inner'' continuous tuple \(\phi\colon[\,0,1\,]\to[\,0,1\,]^{2n+1}\) such that \\
    for any ``target'' function \(f\colon[\,0,1\,]^n\to\R\) and any ``adversarial'' \(h\in\Hcal\), \\
    there exists an ``outer'' function \(g\colon \R\to\R\) such that
    \[
        f(x) = \sum_{i=1}^{2n+1} g\left(\sum_{j=1}^n \lambda_jh_i(\phi(x_j))\right)
        \quad\text{for}\quad x=(x_1,\ldots,x_n)\in[\,0,1\,]^n.
    \]
\end{theorem}

For example, by \cite{DzhenzherFreedman-25}, it is possible to take in Theorem~\ref{t:kart-stab} any countable collection $\Hcal$ of homeomorphisms \(\R^{2n+1}\to\R^{2n+1}\).
It is clear that
\begin{itemize}
    \item if $f$ is continuous, then $g$ can be chosen continuous,
    \item if $f$ is discontinuous, then $g$ must be discontinuous, and
    \item if $f$ is unbounded, then $g$ must be unbounded as well.
\end{itemize}

\begin{remark}[Relevance to NNs and KANs]
    From a machine learning perspective, Theorem~\ref{t:kart-stab} provides a rigorous mathematical guarantee for the structural robustness of KANs operating on non-smooth data.
    In a standard three-layer KAN architecture, the ``inner'' continuous tuple $\phi$ represents the fixed activation functions of the hidden layer, while the ``outer'' function $g$ corresponds to the learnable weights or adaptive univariate functions of the output layer. 
    
    The set of homeomorphisms $\Hcal$ models adversarial reparameterisations or structural perturbations acting directly on the hidden layer. Such perturbations frequently emerge in real-world scenarios due to:
    \begin{itemize}
        \item \textbf{Data Drift and Coordinate Warping:} Systematic shifts or non-linear distortions in the feature extraction pipeline.
        \item \textbf{Adversarial Attacks:} Intentional, mathematically bounded manipulations of the latent representations designed to trick the network.
        \item \textbf{Hardware Constraints:} Precision degradation or analogue drift in specialised neuromorphic hardware executing the inner functions.
    \end{itemize}
    
    By extending KART stability to discontinuous and unbounded target functions $f$, Theorem~\ref{t:kart-stab} guarantees that even if the hidden layer is compromised by an ``adversarial'' homeomorphism $h \in \Hcal$, the network preserves its universal expressivity.
    The distortion can always be completely compensated for by dynamically readjusting the ``outer'' layer $g$.
    Crucially, our result ensures that when the network models physical systems with sharp boundaries, shock waves, or phase transitions (which inherently require discontinuous or unbounded target functions), its capacity to learn and adapt remains invariant under a countable collection of hidden-layer deformations.
\end{remark}

\section{Proof of the main result}\label{s:proof}

The proof of Theorem~\ref{t:kart-stab} is analogous to the proof in \cite{Ismailova-Ismailov}.
In particular, the proof for bounded discontinuous functions relies on dual operator theory, establishing the surjectivity of the corresponding operators via Banach's theorem on the spaces \(\ell_1\) and \(\ell_\infty\); for the unbounded case, where traditional topological frameworks like the Hahn--Banach theorem fail due to the unboundedness of linear functionals, we successfully employed an algebraic framework based on Zorn's lemma to extend the functional over a Hamel basis.
We give the proof in full length, since the forms of KART here and in \cite{Ismailova-Ismailov} differ, which makes our proof slightly less cumbersome.
Though this proof only uses the results of Kolmogorov--Arnold and \cite{DzhenzherFreedman-25} and does not rely directly on the classical proof of KART, we recommend that readers who are not familiar with the subject read the proof for \(n=2\) with \(\lambda_1=1\) and \(\lambda_2=\sqrt{2}\).

\begin{proof}[Proof of Theorem~\ref{t:kart-stab}]
    First, we will use the following common notations.
    For a compact space $X$, $C(X)$ is the Banach space of continuous functions \(X\to\R\).
    Its dual \(C(X)^*\) is the Banach space of Radon measures with bounded variation on the Borel $\sigma$-algebra of $X$.
    Next, \(\ell_1(X) \subset C(X)^*\) is the Banach space of discrete measures \(\nu\) of the form
    \[
        \nu = \sum_{n=1}^\infty c_n\delta_{x_n},
        \quad \norm{\nu} = \sum_{n=1}^\infty \abs{c_n} < +\infty,
    \]
    where \(\{x_n\}\colon \N\to X\), and \(\delta_x\) are the Dirac measures in $x\in X$.
    Finally, \(\ell_\infty(X) = \ell_1(X)^* \supset C(X)\) is the Banach space of bounded functions \(X\to\R\).
    
    Second, note that it is sufficient to construct the ``outer'' function $g$ on a subset of $\R$, since $g$ can be extended to the entire $\R$ arbitrarily (preserving boundedness).

    Let us move to the proof.
    Fix arbitrary rationally independent \(\lambda_1,\ldots,\lambda_n\in\R\).
    For them fixed, fix the ``inner'' continuous tuple \(\phi\colon[\,0,1\,]\to[\,0,1\,]^{2n+1}\) given by the stability of KART under \(\Hcal\).
    Finally, fix an ``adversarial'' reparameterisation \(h\in\Hcal\).
    For \(i = 1,\ldots,2n+1\), define the family of functions \(\Phi_i\colon[\,0,1\,]^n\to\R\) by
    \[
        \Phi_i(x) := \sum_{j=1}^n \lambda_jh_i(\phi(x_j)).
    \]
    Denote by
    \[
        K := \bigcup_{i=1}^{2n+1} \Phi_i[\,0,1\,]^n
    \]
    the compact union of images of those functions.

    \underline{Here, we consider the case of bounded discontinuous ``target'' functions}.
    Define the linear bounded operator \(T\colon C(K) \to C([\,0,1\,]^n)\) by
    \[
        Tg := \sum_{i=1}^{2n+1} g\circ \Phi_i.
    \]
    The stability of KART under $h\in\Hcal$ means that $T$ is surjective.
    By Banach's theorem (see, for example, \cite[Theorem~4.15]{Rudin-1991-fa}), the surjectivity of \(T\) is equivalent to the boundedness from below of $T^*\colon C([\,0,1\,]^n)^*\to C(K)^*$ (which means \(\norm{T^*\mu} \geq c\norm{\mu}\) for some $c>0$).
    Moreover, the conjugate operator has the form
    \[
        T^*\mu = \sum_{i=1}^{2n+1} \mu\circ\Phi_i^{-1},
    \]
    where \(\mu\circ\Phi_i^{-1}\) is the pushforward measure.
    Indeed,
    \[
        \langle T^*\mu, g \rangle = \langle \mu, Tg \rangle = \int_{[0,1]^n} (Tg)(x) \, d\mu(x) =
        \sum_{i=1}^{2n+1} \int_{[0,1]^n} g(\Phi_i(x)) \, d\mu(x) =
        \sum_{i=1}^{2n+1} \int_K g(y) \, d(\mu \circ \Phi_i^{-1})(y).
    \]
    Restricting \(T^*\) to \(\ell_1(K)\), we obtain the operator
    \[
        Z := T^*|_{\ell_1([0,1]^n)} \colon \ell_1([\,0,1\,]^n) \to \ell_1(K)
    \]
    given by
    \[
        Z\delta_x = \sum_{i=1}^{2n+1} \delta_x\circ\Phi_i^{-1} = \sum_{i=1}^{2n+1} \delta_{\Phi_i(x)}.
    \]
    It is clear that $Z$ as the restriction of $T^*$ is also bounded from below.
    Now, consider the dual operator \(Z^*\colon \ell_\infty(K)\to\ell_\infty([\,0,1\,]^n)\).
    Again, by Banach's theorem, \(Z^*\) is surjective.
    Let us explicitly find the form of $Z^*$.
    For any \(g\in \ell_\infty(K)\) and \(x\in [\,0,1\,]^n\),
    \[
        \langle Z^*g, \delta_x \rangle =
        \langle g, Z\delta_x \rangle =
        \left\langle g, \sum_{i=1}^{2n+1} \delta_{\Phi_i(x)} \right\rangle = \sum_{i=1}^{2n+1} g(\Phi_i(x)).
    \]
    Hence
    \[
        Z^*g = \sum_{i=1}^{2n+1} g\circ \Phi_i.
    \]
    Thus the surjectivity of $Z^*$ concludes the proof for bounded ``target'' functions.

    \underline{Now we move to the case of unbounded ``target'' functions.}
    For now, forget the topological space $\ell_1(K)$ and consider the vector space
    \[
        V_K := \vspan{\{\delta_y\}_{y\in K}}
    \]
    of all finite formal linear combinations of Dirac measures on $K$.
    Denote
    \[
        S := \vspan{\{Z\delta_x\}_{x\in[\,0,1\,]^n}} \subset V_K.
    \]
    For a function \(f\colon [\,0,1\,]^n\to\R\),
    define the linear functional \(F_f\colon S\to\R\) by
    \[
        F_f(Z\delta_x) := f(x).
    \]
    Since \(Z\) is bounded from below, its kernel is trivial: \(\Ker Z = \{0\}\).
    This means that the set \(\{Z\delta_x\}_{x\in [0,1]^n}\) is algebraically linearly independent in $V_K$.
    Hence, by Zorn's lemma, we can extend \(\{Z\delta_x\}_{x\in [0,1]^n}\) to a Hamel basis of the entire space $V_K$, and thus extend $F_f$ to the entire \(V_K\) arbitrarily (for example, setting $F_f$ zero on new vectors).
    Finally, defining the ``outer'' function \(g\colon K\to\R\) by
    \[
        g(y) := F_f(\delta_y)
    \]
    yields the required representation.
    Indeed, for any \(x\in[\,0,1\,]^n\),
    \[
        \sum_{i=1}^{2n+1} g(\Phi_i(x)) =
        \sum_{i=1}^{2n+1} F_f(\delta_{\Phi_i(x)}) =
        F_f\left(\sum_{i=1}^{2n+1} \delta_{\Phi_i(x)}\right) = F_f(Z\delta_x) = f(x).
    \]
\end{proof}

Once again, note that in the final step of the proof, one cannot apply the Hahn--Banach theorem instead of Zorn's lemma, since the functional $F_f$ is unbounded for an unbounded function $f$.

\section{Conclusion and further research}\label{s:concl}

In this paper, we have established the stability of the Kolmogorov--Arnold representation theorem under ``adversarial'' reparameterisations of the hidden layer for a significantly broader class of functions, specifically encompassing discontinuous and unbounded multivariate target functions.
By extending previous results that were constrained to continuous settings, we demonstrated that KART preserves its fundamental expressiveness even when subjected to a countable collection of homeomorphisms acting as hidden-layer perturbations.

The results presented open up several compelling directions for future research:

\begin{itemize}
    \item \textbf{Alternative Forms of KART}: We focused on the standard embedding \([\,0,1\,] \to [\,0,1\,]^{2n+1}\).
    A natural next step is to investigate stability within other formulations of KART, such as those pioneered by Sprecher \cite{Sprecher1965, Sprecher96, Sprecher97}, which seek to reduce the total number of required ``inner'' and ``outer'' functions.

    \item \textbf{Shift-Invariant Representations}: Recent Ismailovs' work \cite{Ismailova-Ismailov} explored KART variants where the ``inner'' functions differ strictly by shifts, utilising strict types of rationally independent numbers.
    Deriving stability proofs for this specific variant is highly desirable, as its lower structural complexity makes it far more practical for real-world machine learning implementations, such as Kolmogorov--Arnold Networks (KANs).

    \item \textbf{Constructive Stability and KANs}: While our proof demonstrates the existence of the ``outer'' function \(g\), it remains non-constructive due to its reliance on Zorn's lemma.
    Generalising constructive proofs of KART (e.g., the algorithmic approach by Braun and Griebel \cite{Braun-Griebel-2009}) to accommodate adversarial reparameterisations would bridge the gap between theoretical stability and computational deep learning, potentially leading to provably robust KAN architectures.
    Note that it is interesting to obtain even in the case of continuous functions, where Zorn's lemma does not appear.
\end{itemize}

\printbibliography
% % \printbibitembibliography

\end{document}